\pdfoutput=1

\documentclass[11pt]{article}

\usepackage[final]{acl}

\usepackage{times}
\usepackage{latexsym}

\usepackage[T1]{fontenc}

\usepackage[utf8]{inputenc}

\usepackage{microtype}

\usepackage{inconsolata}

\usepackage{graphicx}
\usepackage{amssymb}
\usepackage{amsmath}
\usepackage{booktabs}
\usepackage{amsthm, amscd, amsfonts, multirow}
\usepackage{makecell}
\usepackage{IEEEtrantools}
\usepackage{algorithm}
\usepackage{algorithmicx}
\usepackage{algpseudocode}

\usepackage{array}
\usepackage{caption}
\usepackage{float}
\usepackage{subcaption}

\usepackage{esdiff}
\usepackage{bbm}

\newtheorem{theorem}{Theorem}

\theoremstyle{definition}
\newtheorem{definition}[theorem]{Definition}
\newtheorem{claim}[theorem]{Claim}

\theoremstyle{remark}

\title{On Support Samples of Next Word Prediction}

\author{
 \textbf{Yuqian Li\textsuperscript{1}\thanks{~~Equal contribution.}},
 \textbf{Yupei Du\textsuperscript{2}\footnotemark[1]},
 \textbf{Yufang Liu\textsuperscript{1}},\\
 \textbf{Feifei Feng\textsuperscript{3}},
 \textbf{Mou Xiao Feng\textsuperscript{3}},
 \textbf{Yuanbin Wu\textsuperscript{1}},
\\
 \textsuperscript{1}School of Computer Science and Technology, East China Normal University,\\
 \textsuperscript{2}Utrecht University, the Netherlands,\\
 \textsuperscript{3}Midea Group,
\\
 \small{
\href{mailto:email@domain}{liyuqian44@gmail.com, y.du@uu.nl, ybwu@cs.ecnu.edu.cn}
 }
}

\begin{document}
\maketitle
\begin{abstract}
%
%

Language models excel in various tasks by making complex decisions, yet understanding the rationale behind these decisions remains a challenge. 
This paper investigates \emph{data-centric interpretability} in language models, 
focusing on the next-word prediction task. 
Using representer theorem, 
we identify two types of \emph{support samples}—those that either promote or deter specific predictions. 
Our findings reveal that being a support sample is an intrinsic property, 
predictable even before training begins. 
Additionally, while non-support samples are less influential in direct predictions, 
they play a critical role in preventing overfitting 
and shaping generalization and representation learning. 
Notably, the importance of non-support samples increases in deeper layers, 
suggesting their significant role in intermediate representation formation. 
These insights shed light on the interplay between data and model decisions, 
offering a new dimension to understanding language model behavior and interpretability.
\footnote{
Our source code is publicly available at
\url{https://github.com/liyuqian44/On-Support-Samples-of-Next-Word-Prediction}.
}

\end{abstract}

\section{Introduction}

Language models make decisions. 
From 
selecting an answer from QA benchmarks
to 
generating reasoning steps for grade school math problems in GSM8k,
language models are applauded and criticized for their decision-making ability.
Like all other AI systems, understanding 
how decisions are made is an important research topic.

Support of decisions can be explored in different parts of language models. 
One common attribution is the patterns of neuron activations (\emph{mechanistic interpretability}): 
some neurons are more sensitive to a specific decision
and some are not. 
The circuit formed by active neurons mechanistically
explains the model's decision \cite{elhage2021mathematical}.
Mechanistic interpretability targets on model parameters
(and hidden states), 
however, it is not the only possible place to 
explore the rationale behind decisions.
From the perspective of compression, 
parameters are the compressed version of training data. 
So, a natural question is whether we can trace model decisions
back to their remote origin, the data.

We follow this point of view and study \emph{data-centric interpretability} of language models.
Concretely, we will focus on decisions about choosing 
the next word (the essential decision of language models), 
and try to answer the question 
``Which samples contribute the most when the model 
decides to predict a token $v$?''. 
For general machine learning problems, 
there are two typical methods: 
one is based on counterfactual argument 
(a sample is important if the decision will change 
if it is removed from the training set 
\cite{pmlr-v70-koh17a}), 
another is based on representation theorems 
(a sample is important if it occupies a large part of
prediction parameters \cite{DBLP:conf/nips/YehKYR18}). 
The first one requires computing Hessians of model 
parameters which could be expensive for analyzing large language models \cite{DBLP:journals/corr/abs-2308-03296}. 
We thus adopt the more efficient representation-based methods. Specifically, 
\begin{itemize}
    \item We illustrate the contribution of different samples
    to model parameters based on a simple representation theorem for the next word predictor. 
    The results remind us of two types of \emph{support samples} (those contribute a lot), 
    one attracts the predictor to predict $v$, 
    while another pushes the predictor away from predicting $v$.
    \item We show that, given a dataset and a model configuration, being a support sample is an intrinsic property of that sample, in the sense that 
    1)  non-support samples have limited influence on learning support samples, 
    and 2) it is possible to predict  (with 80\% accuracy)
    whether a sample is a support one at the very early phase
    of language model training, even without any training.
    \item We find that though non-support samples are less
    noticeable in the next-word predictor, 
    it may play an important role in controlling 
    generalization and representations learning.
    Specifically, 
    we observe that 1) without any non-support 
    samples, learning the predictor will
    suffer from unnecessary overfitting,
    and a small number of non-support samples can greatly alleviate the problem. 
    2) the proportion of non-support samples increases 
    as the layer goes deeper, which means 
    some of them are support samples of 
    (thus contributing a lot to) intermediate representations.
\end{itemize}

\section{Next Word Predictor Representation}

We aim to find important samples to support a language model's
next-word decisions. To measure importance, 
we first connect parameters 
(in particular, prediction heads at the last layer) 
and training samples with a representation theorem, 
then support samples are defined to be those samples 
contributing significantly to parameters. 

Denote $(\mathbf{x},\mathbf{y})$ to be a training sample, 
where $\mathbf{x}=x_1,x_2,...,x_{t-1}$ is a sentence prefix, 
$\mathbf{y}=x_t \in V$ is the target token to predict given the prefix, and $V$ is the vocabulary.
$D=\{(\mathbf{x}_i, \mathbf{y}_i)\}_{i=1}^N$ 
is a set of samples.
We consider a decomposition of language models
with a representation function 
$\phi(\mathbf{x}) \in \mathrm{R}^d$ 
and a token prediction function 
$f(\mathbf{x}) = \arg{\max}_{v} p(v|\mathbf{x})$,
where $p(v|\mathbf{x}) = \frac{1}{Z}\exp{\theta_v^T\phi(\mathbf{x})}$. 
Let $\theta_\phi, \theta_v$ be the parameter of 
$\phi(\cdot)$ and $f(\cdot)$, 
$\theta_V = [\theta_{v_1}; \theta_{v_2};\ldots; \theta_{v_{|V|}}]$
pack all predictor parameters (also called language model heads),
and 
$\theta=[\theta_V;\theta_\phi]$ pack all parameters.

\begin{theorem}[
\citealt{DBLP:journals/jmlr/CrammerS01,DBLP:conf/nips/YehKYR18}
]
\label{thm:rep}
Assume $\theta$ is a stationary point of the loss function 
$L(\theta,D)=-\frac{1}{N}\sum_{i=1}^N\log p(\mathbf{y_i}|\mathbf{x_i})+\lambda\|\theta\|_2^2$,
then $\theta_v$ equals
\begin{equation}
\theta_v = 
   \frac{1}{2N\lambda}\sum_{i=1}^N 
   (\mathbbm{1}(\mathbf{y}_i=v)-p(v|\mathbf{x}_i))\phi(\mathbf{x}_i)
   \label{eq:rep}
\end{equation}
where $\mathbbm{1}(\cdot)$ equals $1$ when the argument is true, $0$ otherwise.
\end{theorem}
The theorem gives a decomposition of the predictor $\theta_v$
using training samples $\phi(\mathbf{x})$.
The coefficients of samples 
$\alpha_i\triangleq \mathbbm{1}(\mathbf{y}_i=v)-p(v|\mathbf{x}_i)$
describe their importance in $\theta_v$.
A large $\alpha_i$ implies a bigger influence, 
and it also implies the sample is difficult to learn:
a large $\alpha_i$ means
$p(v|\mathbf{x}_i)$ is small when $\mathbf{y}_i = v$,
or $p(v|\mathbf{x}_i)$ is large when $\mathbf{y}_i \neq v$,
in both cases the probability of the ground truth token
is small.
\begin{definition}
   The support samples of predicting token $v$ are defined to be
   $S_v=\{(\mathbf{x}_i,\mathbf{y}_i)|\tau \leq |\alpha_i| \}$.
   The support samples of the full language model are
   $S=\cup_v S_v$. 
   The non-support samples are $\bar{S}$.
   \footnote{We set hyperparameter $\tau=0.9$ in this paper as justified in Appendix \ref{section:tau}.}
\end{definition}

We have seen that support samples are roughly 
hard-to-predict samples, 
here we can also inspect the opposite side 
by considering the relationship between non-support samples
and easy-to-predict samples.
Specifically, we consider 
$M=\{(\mathbf{x}, \mathbf{y})|
f(\mathbf{x})=\mathbf{y}, p(v|\mathbf{x}) \geq \gamma\}$,
which are samples correctly predicted after training,
and defined to be ``memorized'' by language models 
\cite{DBLP:conf/nips/TirumalaMZA22}.
The following fact shows that the memorized samples 
are a kind of non-support samples.
\footnote{The original definition of memorized samples in
\cite{DBLP:conf/nips/TirumalaMZA22} sets $\gamma=0$. 
We can see a subtle rationale for introducing 
the confidence threshold $\gamma$:
the claim here says that even a sample's correct label 
can be predicted, if the prediction is not confident enough, 
it could still be a support one.}
\begin{claim}
\label{thm:mem}
If $\gamma, \tau \geq 0.5$, the non-support samples $\bar{S}$ contains $M$.
\end{claim}
\begin{proof}
When $\gamma \geq 0.5$,  $p(\mathbf{y}|\mathbf{x}) \geq \tau$
implies $f(\mathbf{x})=\mathbf{y}$,
and $M=\{(\mathbf{x}, \mathbf{y})| p(v|\mathbf{x}) \geq \gamma\}$.

At the same time, for any $(\mathbf{x}, \mathbf{y}) \in M$,
if there exists a $v$ such that 
$(\mathbf{x}, \mathbf{y})\in S_v$, 
\begin{itemize}
    \item if $v=\mathbf{y}$, $\alpha = 1-
    p(\mathbf{y}|\mathbf{x})< 0.5 \leq \tau$,
    \item if $v\neq\mathbf{y}$, since 
    $p(\mathbf{y}|\mathbf{x}) > 0.5$,
    $\alpha = p(v|\mathbf{x}) < 0.5 \leq \tau $.
\end{itemize}
There are contradictions $\alpha < \tau$ in both cases, 
thus for all $v$, $(\mathbf{x}, \mathbf{y}) \notin S_v$.
\end{proof}


\section{Support Samples at First Glance}\label{sec:first_glance}

We now profile support samples.
Our exemplar language model is trained from scratch
with GPT-2 architecture 
($117$M parameters, with $l=12$ Transformer layers and 
hidden size $d=768$)
and wikitext-2 dataset ($2.37$M training samples).
The vocabulary size $|V|=50257$.
We will investigate models with more parameters 
($345$M, $774$M, $1.5$B) and larger training set 
(wikitext-103 with $117$M samples)
in specific experiments, and report results on larger models in Appendix \ref{section:scaleup}.


\begin{figure}
  \includegraphics[width=\columnwidth]{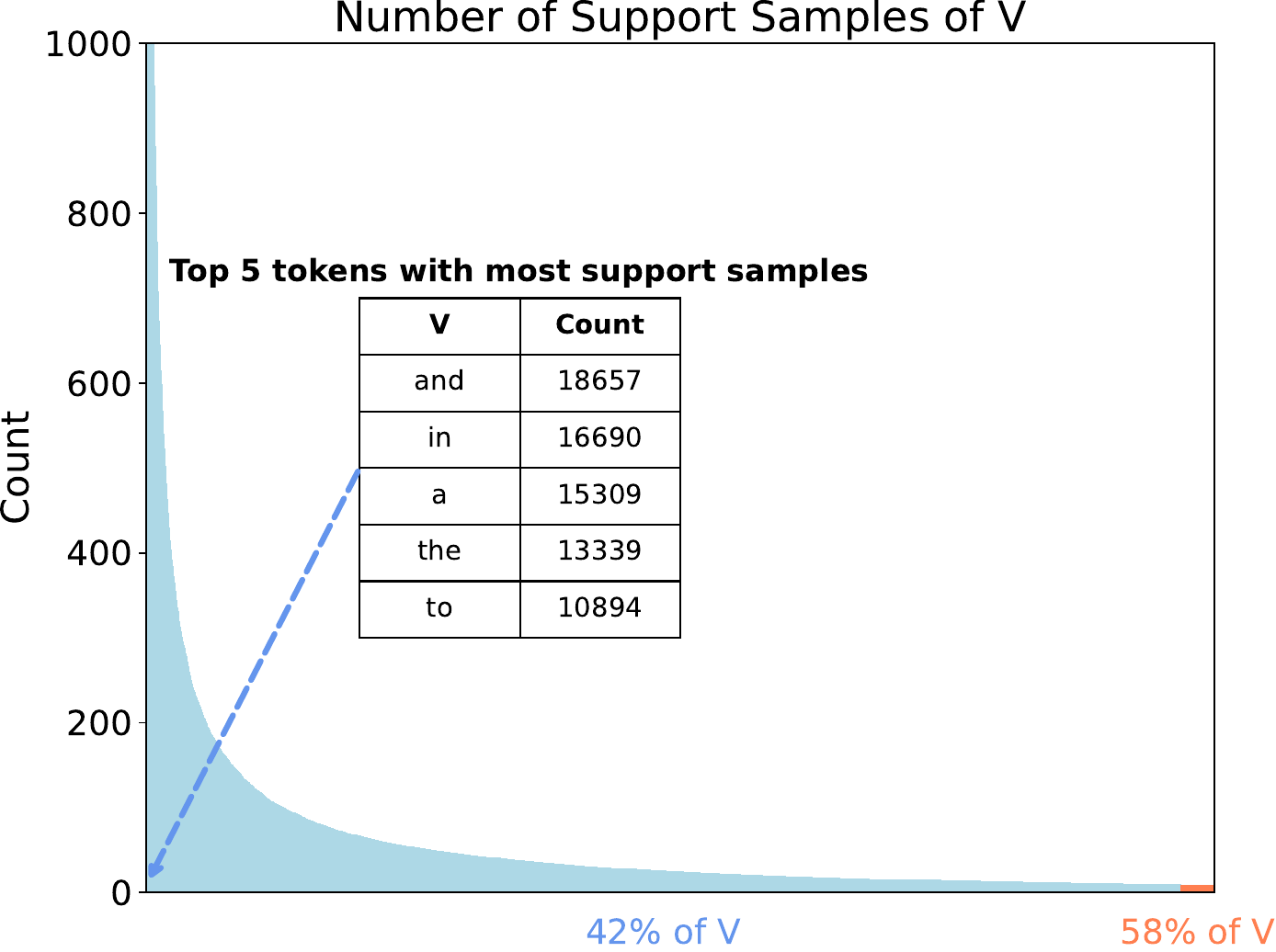}
  \caption{Number distribution of support samples for each token $v$.
  Blue bars are the top $42\%$ tokens with the most support samples. Orange bars are the remaining $58\%$ tokens which have less than $10$ support samples. 
  The table shows the number of the top 5 tokens. }
  \label{fig:v_support_number}
\end{figure}

\paragraph{Number of support samples.}
Our first observation is that, 
with sufficient training (loss converges) 
and standard model selection 
(best validation set performances),
among $2.37$M samples, 
there are $1.29$M support samples ($54\%$).
The proportion is remarkably high:
more than half of the samples are
important for the next-word predictor.
The larger amount of support samples
implies fewer patterns are discovered during training:
the predictor fails to make the correct decision
by keeping a few representative samples.
\footnote{Therefore, we may also have a 
discuss on the definition of memorized samples:
in fact, akin to the memory of a computer, we can imagine the support samples (rather than non-support samples or 
previously defined memorized samples (Claim \ref{thm:mem})) 
are stored (or memorized) in
parameters, and they will be looked up later for token prediction.}

By depicting the number of support samples 
for individual $v\in V$ (Figure \ref{fig:v_support_number}),
we further find the distribution is highly screwed:
very few $v$ contribute a large number of support samples.
It implies that though the overall support samples are
many, prediction patterns of most tokens are clear
(i.e., $58\%$ tokens have support samples less than $10$).
We additionally provide a concrete example of the support and non-support samples for a specific $v$ in Appendix \ref{section:example}.


\begin{figure}[t]
  \includegraphics[width=\columnwidth]{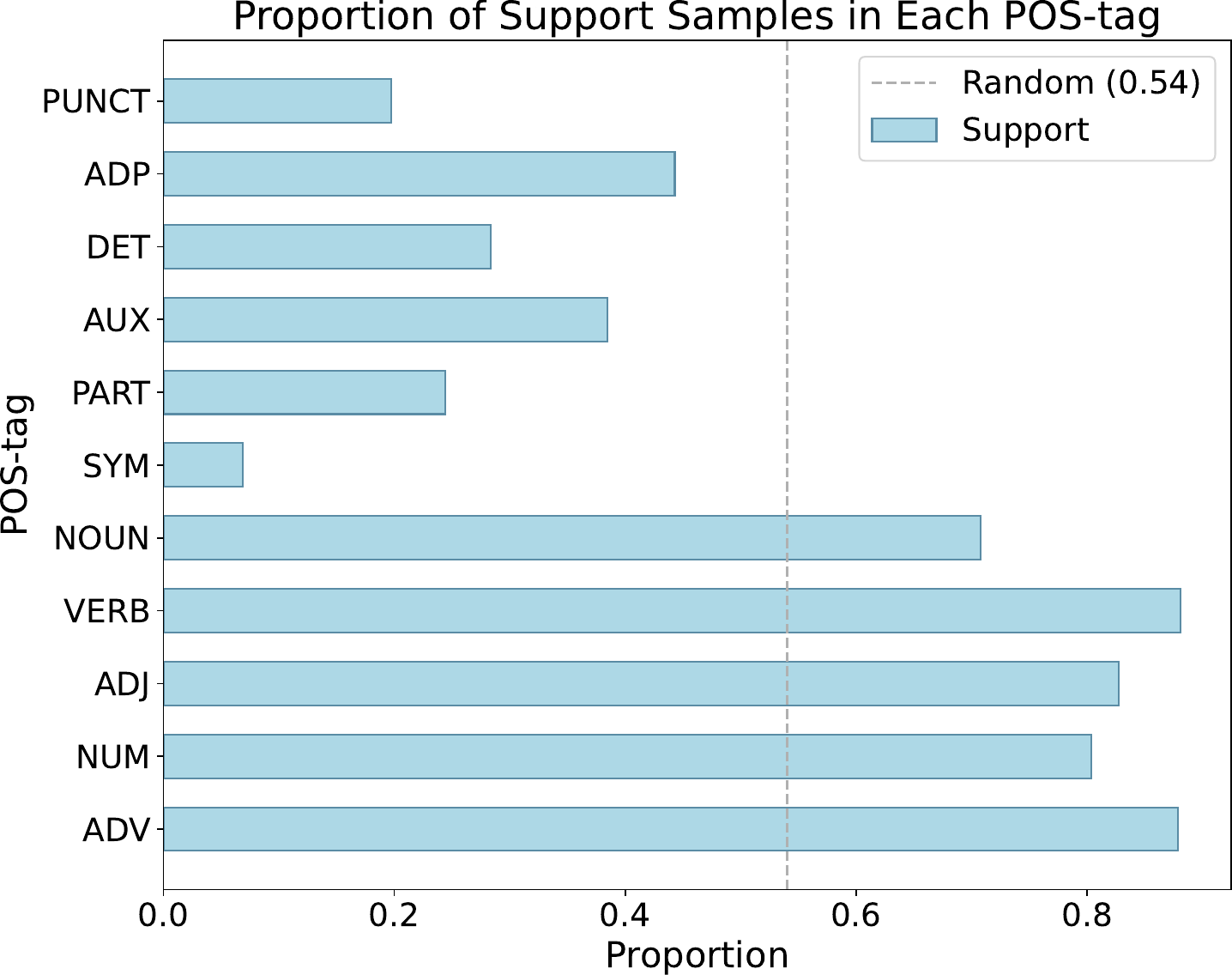}
  \caption{Proportion of support samples in different POS tags. ``Random'' means that a random sampling 
  on the full dataset will pick $54\%$ (which is 
  the proportion of support samples) support samples. 
  The figure shows that the proportions of support samples in 
  two POS-tag groups
  significantly differ from the random proportion.}
  \label{fig:postag}
\end{figure}

\paragraph{POS tags of support samples.} 
To inspect support samples more closely,
we roughly annotate POS tag of the next token 
$\mathbf{y}$ given its prefix $\mathbf{x}$.
\footnote{https://spacy.io/},
and see whether samples with different target POS tags tokens 
will have different probabilities to be support samples.
Figure \ref{fig:postag} shows that  
$86\%$ of verbs are support samples
while the proportion is only $20\%$ for punctuations.
In general, support samples 
are more prevalent in semantically rich POS tags 
(e.g., ADV, NUM, ADJ, VERB, NOUN), 
while less common in semantically light tags 
(e.g., PART, AUX, DET, ADP, PUNCT).
It indicates that relatively, the predictor keeps more
information about verbs than punctuations in 
its parameters.

\paragraph{Two types of support samples}
By revisiting the representation of $\theta_v$,
we can group support samples into two types. Specifically,
Equation \ref{eq:rep} can be rewrite as
(dropping $(N\lambda)^{-1}$),
\begin{equation}
    \sum_{\mathbf{y}_i=v}
   (1-p(v|\mathbf{x}_i))\phi(\mathbf{x}_i) + 
    \sum_{\mathbf{y}_i \neq v}
   (-p(v|\mathbf{x}_i))\phi(\mathbf{x}_i)
\end{equation}
For each target token $v$, we define support samples 
from the first summation to be Type-1, and
the second summation Type-2.
Namely, Type-1 are support samples with the same target token as $v$,
but their prediction confidence is low.
Type-2 are samples with different target tokens, 
but they are confident in predicting the wrong target $v$.

To understand the roles of Type-1 and Type-2 samples,
we conduct the following counterfactual argument:
what if we remove (subtract) one support sample type of $v$ from 
all prediction parameters $\theta_{V}$.
Let $(\mathbf{x}',\mathbf{y}')$ be a testing sample.
Empirically, we assume $\phi(\mathbf{x}_i)^T\phi(\mathbf{x}') \geq 0$
for all $i$.

If we subtract Type-1 samples of $v$, 
the logit $\theta_v^T\phi(\mathbf{x}')$ will be smaller
since a positive term 
$\sum_{\mathbf{y}=v}(1-p(\mathbf{y}_i|\mathbf{x}_i))
\phi(\mathbf{x}_i)^T\phi(\mathbf{x}')$ is dropped.
With a softmax activation, the probability of predicting $v$ 
could be greatly decreased, and the model resists predicting $v$.
Similarly, when subtracting Type-2 samples, 
a negative term is dropped, and the model tends to predict $v$.
\footnote{We also simply (and empirically) ignore the operations' potential complex influences on other predictors (Figure \ref{fig:v_support_number}).}
Therefore, we would think that the existence of Type-1 support samples 
pull the predictor towards predicting $v$,
while Type-2 samples push the predictor away from predicting $v$.

Table \ref{tab:type} lists a concrete example for which we 
conduct the above subtracting experiments.
The results show that, for these two 
when removing Type-1, predictions on $v$ fail,
and predictions on other samples can be kept
(even improved, compared with random subtracting:
$3.35$ vs. $3.75$, and $3.28$ vs. $3.80$).
When removing Type-2, 
predictions on $v$ are perfect,
while other samples are negatively affected.
It shows that holding Type-2 support samples 
is crucial for the predictor, which is 
new to the perplexity-based data selection principle
(only Type-1 samples are important).

To further understand Type-2, we draw a network to illustrate
connections of next-token targets.
In Figure \ref{fig:type2_graph}, 
we link tokens according to Type-2 supporting relations: 
each directed edge from $v$ to $u$ indicates that a sample 
with $v$ as the target token is the Type-2 support sample of $u$. 
We also list hubs with most in-degrees and out-degrees.


\begin{table}
  \centering
\small
  \begin{tabular}{lcc|cc}
    \toprule
   & \multicolumn{2}{c|}{$v$=``hens''} & \multicolumn{2}{c}{$v$=``ction''}  \\
    \textbf{Removed Set} & \textbf{full-set} & $\mathbf{y}=v$ & \textbf{full-set} & $\mathbf{y}=v$   \\
    \midrule
     -$\emptyset$   & 3.28 & 0.24  & 3.28 & 0.56          \\
     -$S_v$    & 4.45 & 2.37  & 3.52 & 1.50         \\
     -Type-1 of $S_v$    & 3.35 &  16.73  & 3.28 & 33.09          \\
     -Type-2 of $S_v$ & 4.75 &0.00  & 3.90 & 0.00          \\
     -random   & 3.75 & 0.27 & 3.80 & 0.51          \\
    \bottomrule
  \end{tabular}
  \caption{Removing support samples of $v$ = ``hens'' and ``ction''.
  The number of support samples $|S_v|=10$, and half of them 
  are Type-1.
  When removing the samples, we subtract them from all $\theta_v$.
  We consider four sets to remove, $S_v$, Type-1/Type-2 of $S_v$,
  and a random subset with the same number.
  We report the loss of all training samples (full-set)
  and loss of a subset which contains samples with
  $v$ as the target token ($\mathbf{y}=v$).
  }
  \label{tab:type}
\end{table}


\begin{figure}[t]
  \centering
  \begin{minipage}[t]{\linewidth}
    \centering
    \includegraphics[width=\linewidth]{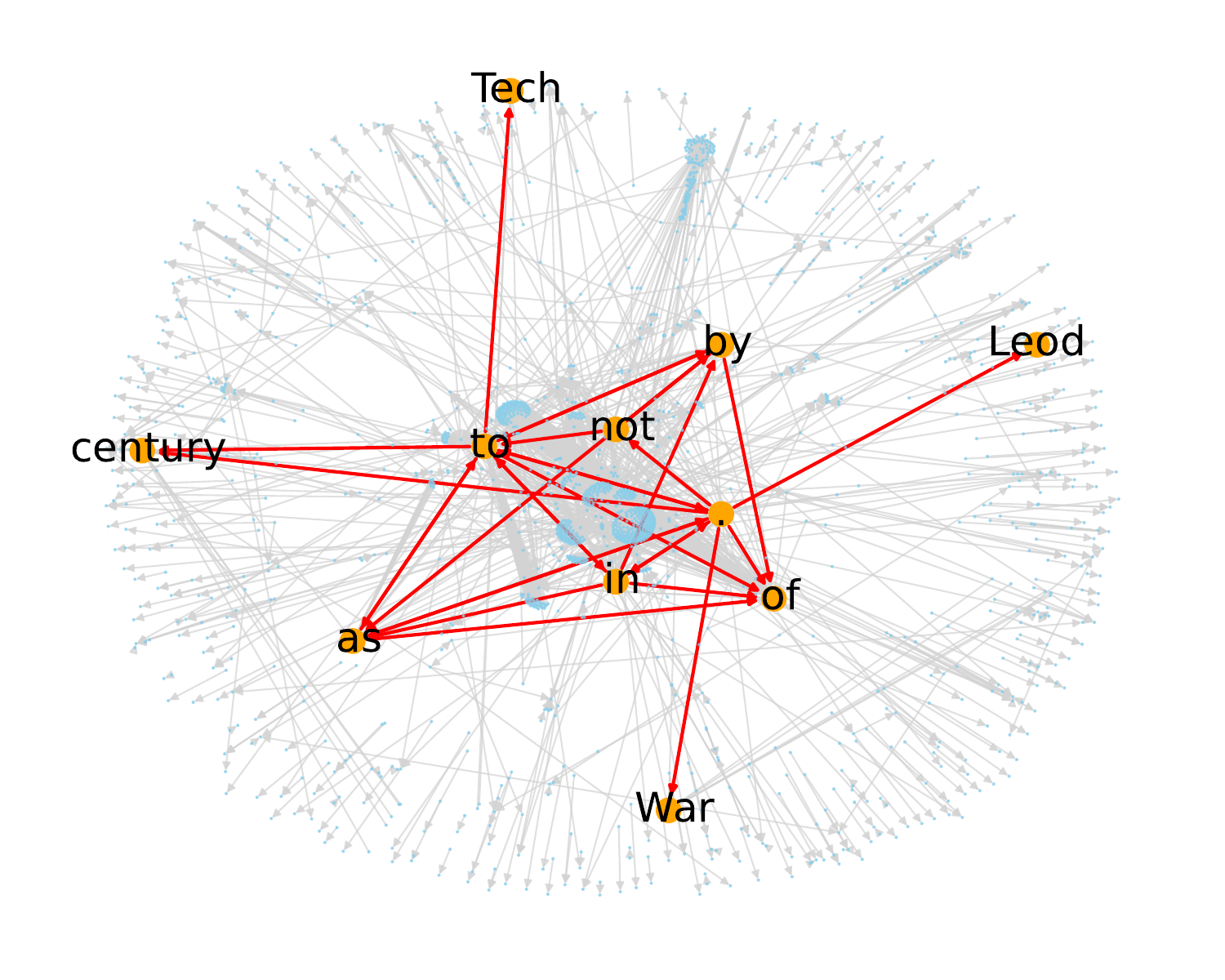} 
  \end{minipage}
  \begin{minipage}[t]{\linewidth}
  \small
    \centering
    \begin{tabular}{cccccc}
      \toprule
      $v$ & \textbf{,} & \textbf{@} & \textbf{and} & \textbf{.} & \textbf{in} \\
      out-degree & 30 & 28 & 27 & 22 & 17 \\
      \midrule
      $v$ & \textbf{=} & \textbf{of} & \textbf{s} & \textbf{,} & \textbf{@} \\
      in-degree & 242 & 153 & 144 & 134 & 84 \\
      \bottomrule
    \end{tabular}
  \end{minipage}

  \caption{
  An illustration of token relations according to Type-2. 
  The number of nodes is $1690$ (3\% of $V$, not all tokens have Type-2 support samples). 
  The number of edges is $1993$.
  We lists tokens with top k(=5) out-degree and in-degree in the network.}
    \label{fig:type2_graph}
\end{figure}
\section{Importance of Support Samples}\label{sec:importance_of_support}

Support samples contribute the most to $\theta_v$ (LM heads). This leads us to pose a question: \emph{If the support samples are defined to be the most important samples, can we only use them (ignore those non-support samples) during training?}

We examine alternative ways of removing non-support samples: removing with sampling 1) \textbf{hard}: removing all non-support samples; 2) \textbf{random}; uniformly removing the same number of samples as the non-support; 3) \textbf{soft}: using the coefficient $\alpha_i$ as the probability for sampling, meaning that some non-support samples will be retained. We experiment with two training configurations: \textbf{only training LM heads} (other parameters inherited from the original trained LM) and \textbf{full model training}. The result is shown in Table \ref{tab:removing non-support}.

For training LM heads, the answer is: Yes, we can remove non-support samples (around 40\%-50\% of the samples) while keeping the same performances, with weighted sampling or even random sampling. However, for training full models, the answer is No. 

An interesting phenomenon is that soft sampling outperforms random sampling for training LM heads (fixed representation space), but the opposite occurs when training full model. The former tells non-support samples are not necessary for learning heads (in the sense of soft removing). The latter implies non-support samples are important for learning representations (parts of the model beyond the heads). Otherwise, the model trained only with support samples will perform like only head tuning with support samples. A deeper exploration of this point is presented in Section 6. 

Curiously, for tuning heads, hard sampling performs worse than soft sampling.(more or even exclusively support samples does not lead to better performance, and meanwhile random sampling with fewer support samples performs worse.) Observing the changes in support samples, we find that although hard sampling slightly reduces original support samples (1.29$\to$1.18), it turns more unseen non-support samples into support(1.29$\to$1.90). This makes the predictor's hyperplane unclearer (more support samples), and leads to a performance drop. 
It means without any non-support samples, learning the predictor will suffer from unnecessary overfitting to support. A small number of non-support samples can greatly alleviate the problem (soft better than hard), while support samples are still key for tuning heads (soft better than random). Regarding full model training, the lack of non-support samples poses a more significant issue, thus random sampling surpasses soft sampling, because non-support samples are more important for parts of the model beyond the heads as mentioned in the previous paragraph.


\begin{table*}
  \centering
  \begin{subtable}[t]{\linewidth}
  \small
    \centering
    \begin{tabular}{lccc}
      \toprule
      \multirow{2}{*}{\textbf{Method}} & \multirow{2}{*}{\textbf{ Test Loss}} & \multicolumn{2}{c}{\textbf{\# Support Samples(M)}} \\
      \cmidrule(lr){3-4}
      & & new training set (after removal) & original training set \\
      \midrule
      \textbf{-}    & 5.08    & 1.29 &    1.29   \\
      hard    & 5.64      & 1.29$\to$1.18&  1.29$\to$1.90  \\ 
      soft    & 5.13   &  0.87$\to$0.87&   1.29$\to$1.33   \\
      random    & 5.18   & 0.70$\to$0.74 &   1.29$\to$1.38 \\\bottomrule
    \end{tabular}
    \caption{only training LM heads}
    \label{tab:first}
  \end{subtable}
  \hspace{0.05\linewidth}
  \begin{subtable}[t]{\linewidth}
    \centering
    \small
    \begin{tabular}{lccc}
      \toprule
      \multirow{2}{*}{\textbf{Method}} & \multirow{2}{*}{\textbf{Test Loss}} & \multicolumn{2}{c}{\textbf{\# Support Samples(M)}} \\
      \cmidrule(lr){3-4}
      & & new training set (after removal) & original training set  \\
      \midrule
      \textbf{-}      & 5.08   & 1.29 &   1.29    \\
      hard      &6.57     &  1.29$\to$1.28 &    1.29$\to$2.36 \\
      soft      &5.69   &  0.87$\to$1.01&   1.29$\to$1.76    \\
      random      &5.47  & 0.70$\to$0.73&  1.29$\to$1.45  \\\bottomrule
    \end{tabular}
    \caption{full model training}
    \label{tab:second}
  \end{subtable}
  \caption{Performance of the predictor and Change of support samples
  when removing non-support samples from original training set with various sampling methods. \textbf{-}: no removing (original training set). \textbf{hard}: removing all non-support samples; \textbf{random}: randomly removing the same number of samples as the non-support; \textbf{soft}: using the coefficient $\alpha_i$ as the probability for weighted sampling, meaning that some non-support samples will be retained. We count support samples on two types of dataset: one is new training set (after removal) as the initial number of support samples varies for each method, another is original training set.}
  \label{tab:removing non-support}
\end{table*}

See Table \ref{tab:removing non-support} again, a pattern seems to emerge: the fewer the support samples, the lower the loss. We already see that the learned next-word predictor has many support samples. The next question is: \emph{Is it possible to make the hyperplane clearer by reducing some support samples?}

We try more data, various regularization methods (L2 and dropout). Figure \ref{fig:datasize} first shows that the trends of loss and support proportion are largely consistent. And more data can help reducing support samples to make the hyperplane clearer and the predictor better, while L2 and embedding dropout fail.

\begin{figure*}[ht]
  \includegraphics[width=0.33\linewidth]{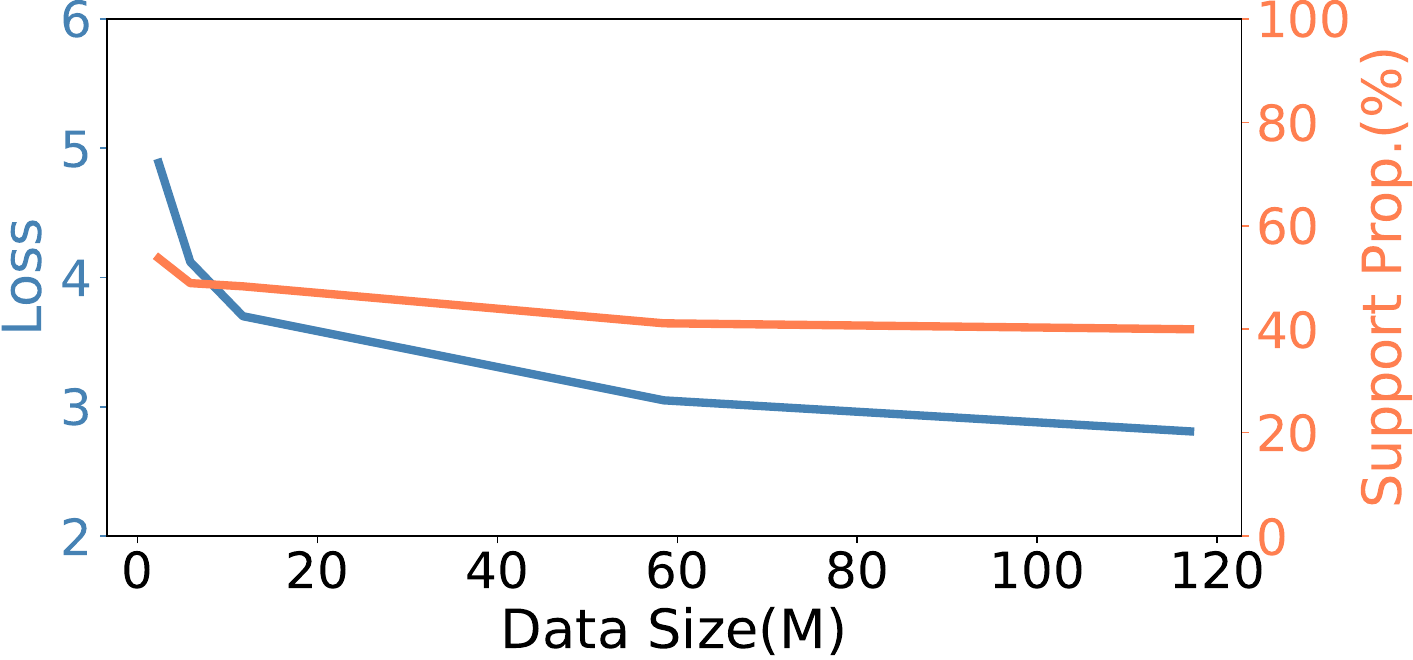} \hfill
  \includegraphics[width=0.33\linewidth]{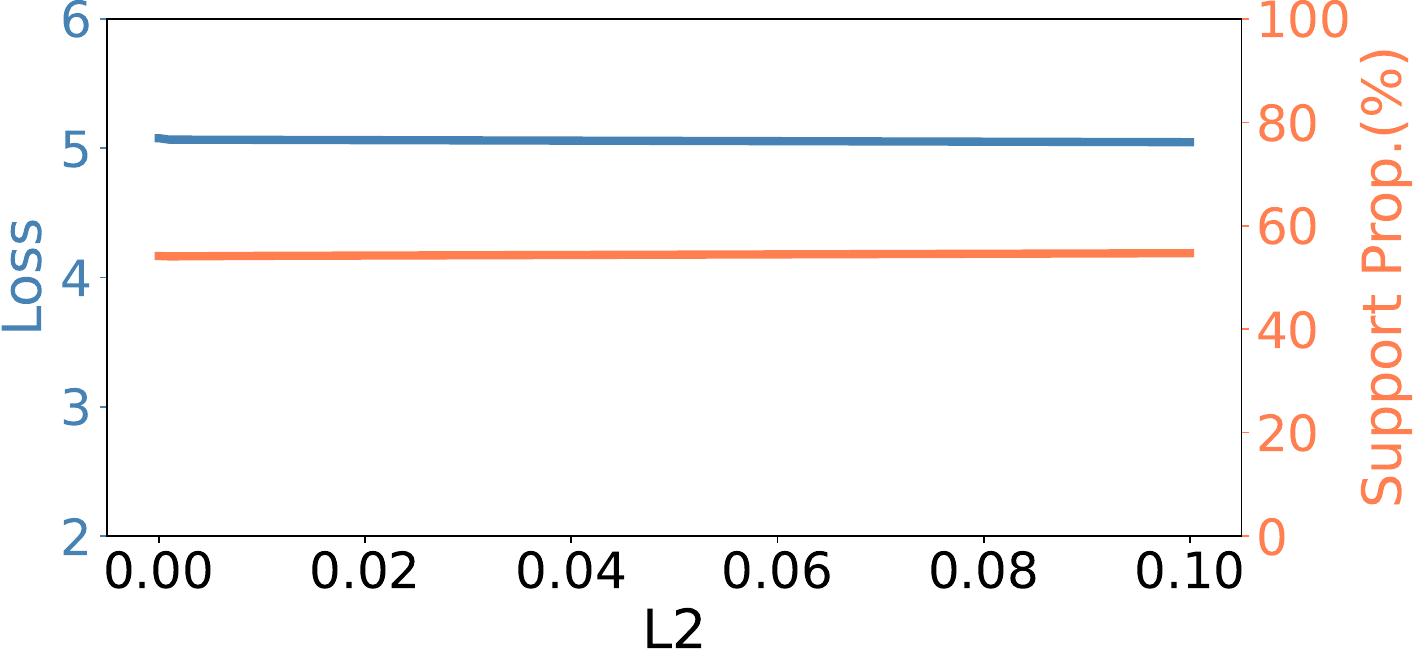}\hfill
  \includegraphics[width=0.33\linewidth]{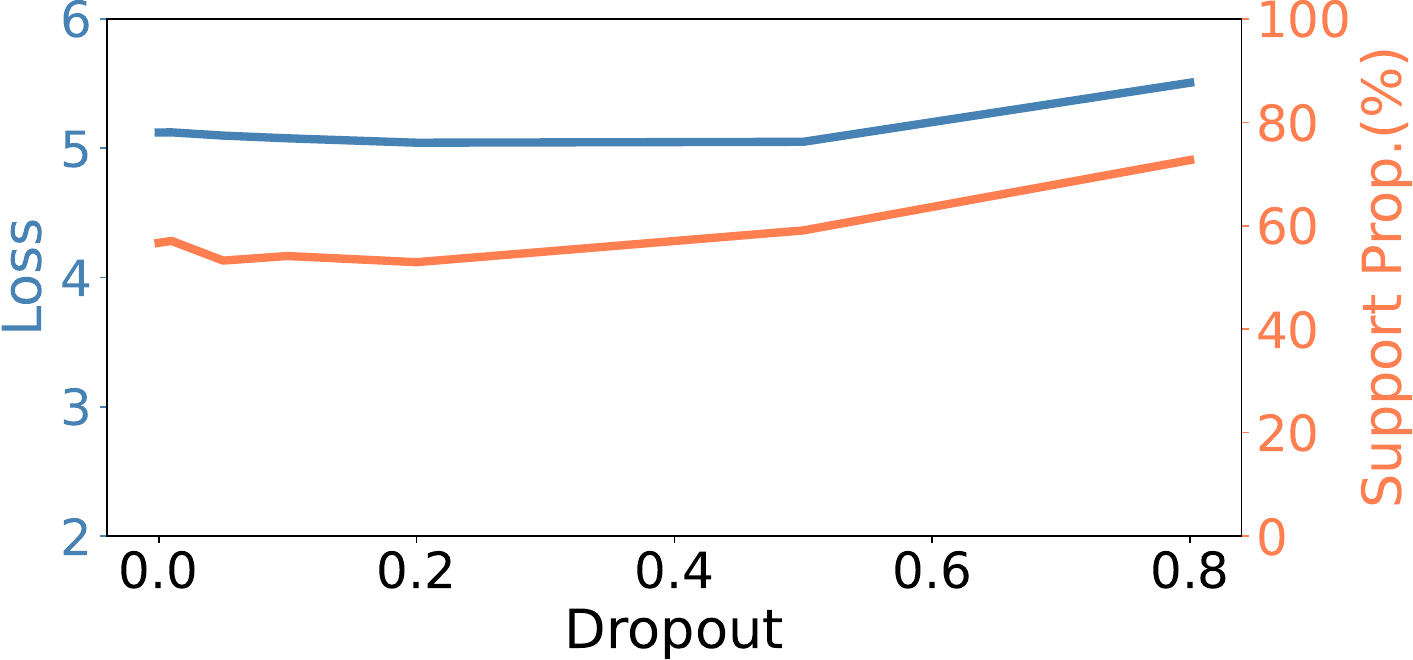}
  \caption {The change of Loss and Support proportion on different training sets. \textbf{Data Size}: we sample subsets of different sizes from wikitext-103.\textbf{L2}: we set the weight decay parameter to adjust L2 regularization during training. \textbf{Dropout}: we try different dropout rates at the embedding layer.}
  \label{fig:datasize}
\end{figure*}

In conclusion, we get 1) For tuning heads, support samples are key, but some non-support samples are also necessary. We can remove non-support samples with soft sampling. 2) Non-support samples are important for learning representations. 3) Fewer support samples indicate the clearer hyperplane and better performance of the predictor. More data vitals.
\section{Predicting Support Samples}\label{sec:predicting_support}


We have another discovery when training with subsets (not shown in Table \ref{tab:removing non-support}): the loss on original support samples are almost the same for models trained with and without non-support samples ($6.17$ vs. $6.43$). \footnote{The loss of hard sampling (training only with support samples) on the original samples is $6.43$ and of random sampling is $6.17$. In fact, the original model is $5.30$,
but the comparison may not be fair since the size of the training
sets are different. We thus take the random sampling $6.17$ 
as our baseline.  }

The result implies that non-support samples 
have limited influence on better fitting support samples,
and it roughly suggests that the hardness of support samples 
might be ``intrinsic''.
To make the problem clear, we cast an interesting question,
\emph{can we predict whether a sample is a support sample or not without training the language model?}

Surprisingly, we have a positive answer.
We will show that given initial parameters, 
with simple classifiers (linear and MLP) and 
a limited number of training samples 
(annotations of support and non-support samples), 
one could recognize samples with high accuracy.
It may confirm our ``intrinsic'' property intuition.
The following are details.


Denote $\zeta(\mathbf{x}, \mathbf{y}, \theta^{(t)})$ 
to be a feature map which extracts features
of a sample $(\mathbf{x}, \mathbf{y})$ according to 
model parameter $\theta^{(t)}$ at checkpoint $t$
(especially, we are interested in the initial state $\theta^{(0)}$).
$h_\eta(\cdot)$
is a binary classifier (with parameter $\eta$)
which takes a feature vector
$\mathbf{z}=\zeta(\mathbf{x}, \mathbf{y}, \theta^{(t)})$ 
and outputs $\mathbf{s} \in \{0, 1\}$
indicating whether
the input is a support sample or not.
To learn the classifier, we need a training set $E=\{(\mathbf{z}_j, \mathbf{s}_j)\}_{j=1}^K$.
We experiment with the following configurations.

\paragraph{The feature map $\zeta$}
We consider three types of features at checkpoint $t$,
\begin{itemize}
  \item the last hidden vectors
  $\phi_{\theta^{(t)}}(\mathbf{x}) \in \mathrm{R}^d$,
  which is $\phi(\mathbf{x})$ evaluated with 
  checkpoint $\theta^{(t)}$. 
  Recalling that support samples are defined 
  according to $p(v|\mathbf{x})$ which is
  a generalized linear model of
  $\phi(\mathbf{x})$,
  last hidden vectors seem to be a natural
  and effective choice.
  \item concatenation of all hidden vectors
   $[\phi_{\theta^{(t)}}^1(\mathbf{x});
   \phi_{\theta^{(t)}}^2(\mathbf{x}); 
   \dots; \phi_{\theta^{(t)}}^l(\mathbf{x})]
   \in \mathrm{R}^{dl}$,
   where $\phi_{\theta^{(t)}}^j(\mathbf{x})$
   is the hidden vector of layer $j$.
   The feature includes intermediate representations.
   \item gradient features $\nabla_\theta p(\mathbf{y}|\mathbf{x})|_{\theta^{(t)}}$.
   Instead of directly using the gradients, which 
   are in the extremely high dimension space,
   we perform a random projection of the gradient 
   vector to a low dimension (4096). 
   \footnote{The time complexity of gradient features(gradients are only computed offline once) and the space complexity(gradients are projected to low dimension 4096) are acceptable.}
\end{itemize}

\paragraph{Classifier}
We consider two classifiers,
\begin{itemize}
    \item linear classifiers, 
    $h_\eta(\mathbf{z}) = \eta^T\zeta(\mathbf{x}, \mathbf{y}, \theta^{(t)})$.
    The learnable parameter $\eta$
    equals the dimension of input features.
    \item MLPs,
    $h_\eta(\mathbf{z}) = w^TW_2\sigma(W_1\zeta(\mathbf{x}, \mathbf{y}, \theta^{(t)}))$,
    where $\sigma$ is a non-linear activation function, 
    and $\eta=[w; W_1; W_2]$.
    We try different model capacities by
    varying the size of $W_1, W_2$.
\end{itemize}

\paragraph{Training set}
We collect training set $E$ for this classification task
from the training samples of the language model $D$. 
We first train the language model with $D$,
and get annotations of support/non-support samples according 
to the last checkpoint.
We split the set to training, validation and testing set 
with 8:1:1 ratio.
\footnote{In the experiment, we use wikitext-2 set. The binary classifier has 1.90M training samples, 0.24M validation samples, and 0.24M testing samples.}
We also test smaller versions of the training set.

\paragraph{Checkpoints} 
We extract features from different checpoints of 
the language model.
Two special checkpoints are,
\begin{itemize}
    \item checkpoint $\theta^{(0)}$ which is the initial language model without any training.
    \item the last checkpoint  $\theta^{(T)}=\theta$
    which is the parameter defining the gold standard 
    annotation of support/non-support samples.
    We may think the performances of classifiers 
    at the last checkpoint are performances upperbounds:
    a classifier can see all necessary information
    for recognizing support samples.
\end{itemize}
As $\theta^{(0)}$ could be hard,
we are also interested in the performances at the early
checkpoints of training.

The experimental results (Figure \ref{fig:predict_accuracy}) 
show that with MLP and gradient feature, the prediction 
accuracy can reach 80\% without training, just 5\% shy of 
the upper limit (last checkpoint). The classifier of the 
last checkpoint still cannot achieve 100\% accuracy, 
suggesting that the cut between support and non-support samples is not clear. 
If we allow extracting features from early training phases,
the prediction accuracy can be better.

Performance improves with more data and more parameters (MLP with increasing intermediate dimension) but eventually levels off. The upper limit of gradient feature is around 80\%, of last hidden state is about 65\%, and of all hidden states is 66\%. Gradient outperforms hidden states, particularly in the early training stages. Even with more parameters, hidden states still cannot catch up.

In brief, we successfully predict whether a sample is a support sample without training (80\% accuracy) only using a simple classifier (MLP) with a randomly projected gradient feature.

\begin{figure*}[t]
  \includegraphics[width=0.32\linewidth]{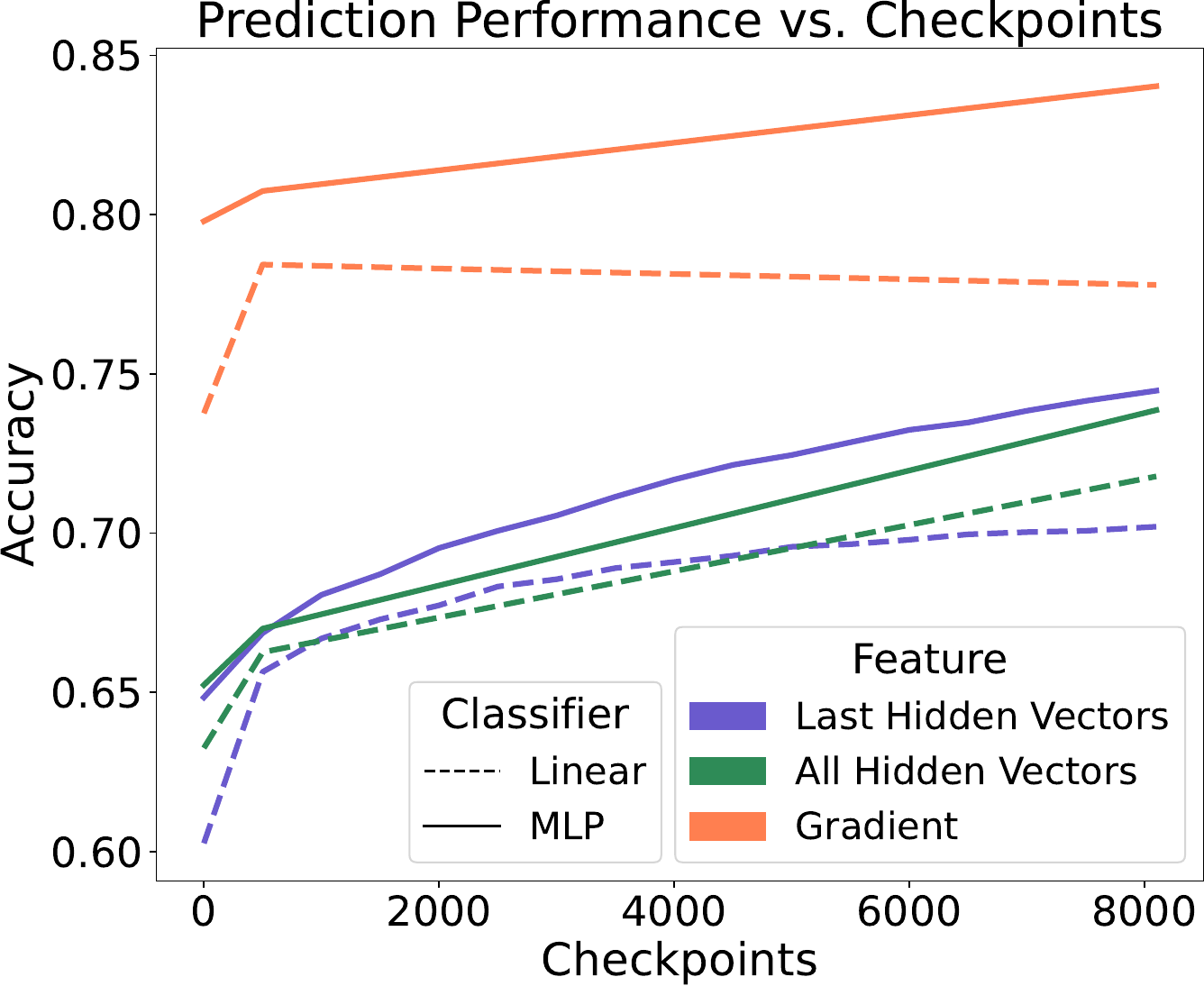} \hfill
  \includegraphics[width=0.32\linewidth]{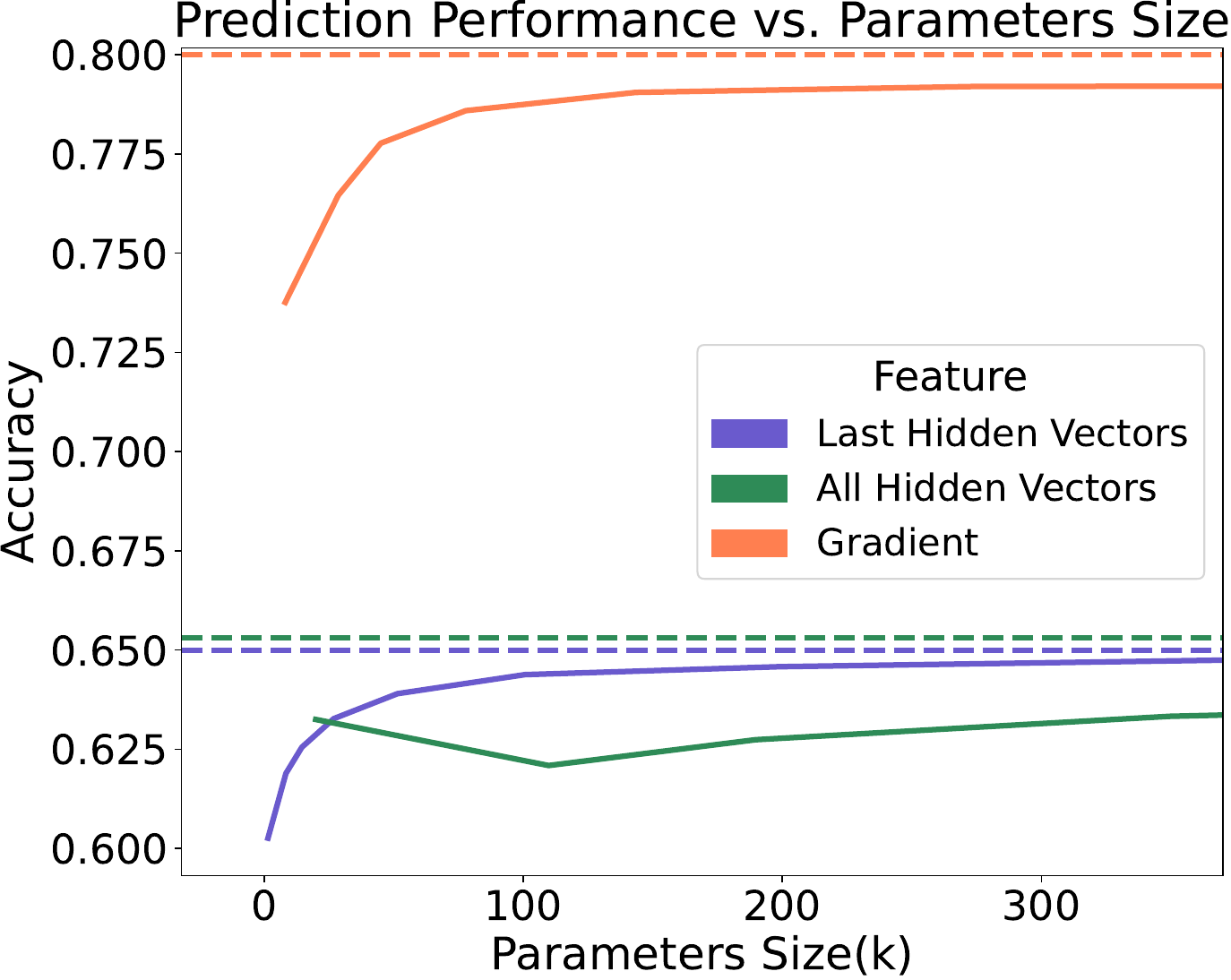} \hfill
  \includegraphics[width=0.32\linewidth]{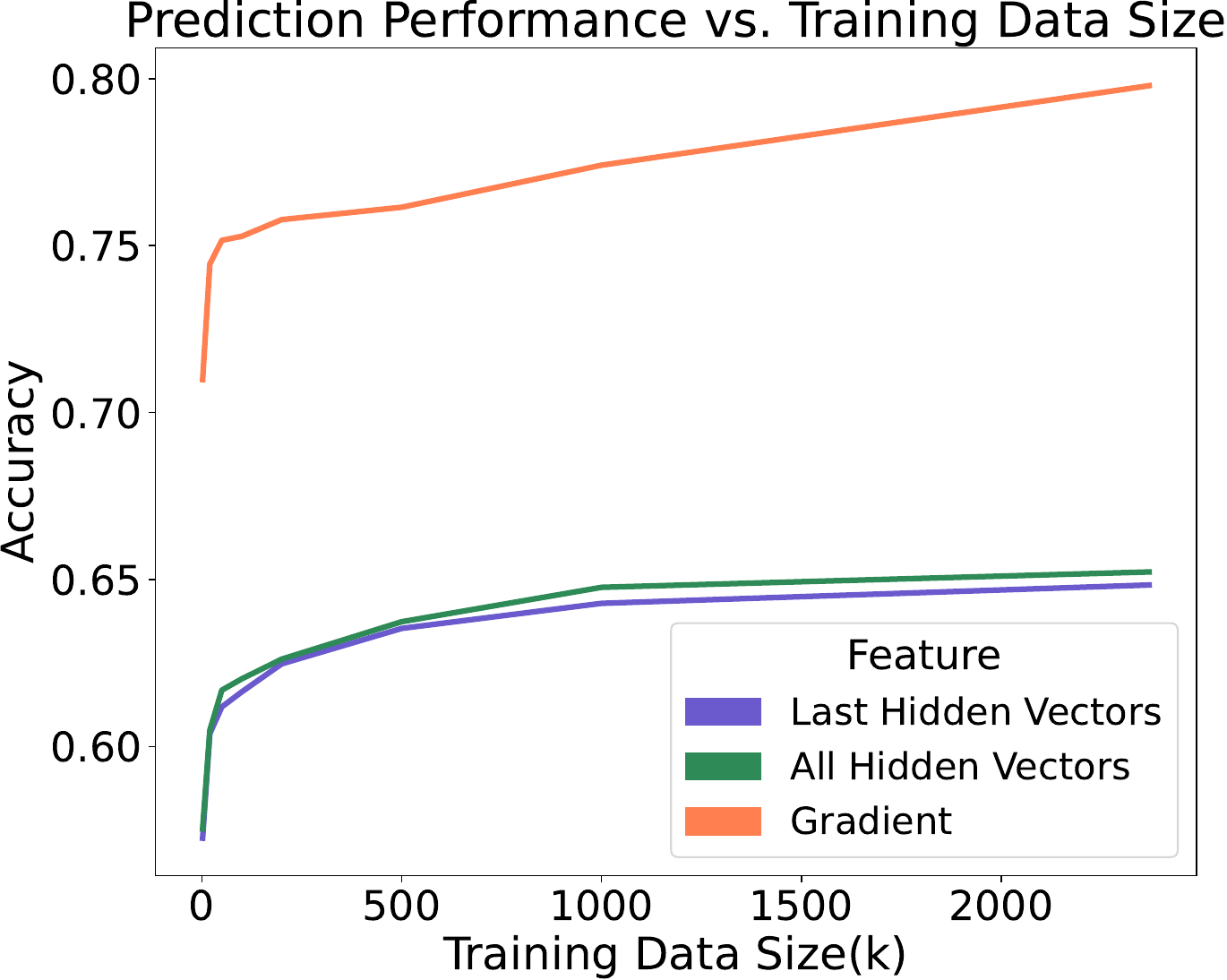}
  \caption {Performance of the classifier under various variables. \textbf{Left}: Performance changes across different training stages. The accuracy achieves nearly 80\% even without training, and only improves to 85\% at the last/best step. \textbf{Middle}: Performance improves with more parameters (MLP with increasing intermediate dimension) but eventually levels off. 
  The upper limit of gradient is around 80\%, of last hidden state is about 65\%, and of all hidden states is 66\%. 
  Due to the varying feature dimensions (last hidden state: 768; all hidden states: 9984; gradient: projected to 4096) and significant differences in parameter counts, only the portion within 350k parameters is shown in the figure for clarity. In practice, all hidden states can achieve 66\% with more parameters. \textbf{Right}: Performance enhances with the growth of training data. All three figures demonstrate that, in terms of features, gradient outperforms hidden states, particularly in the early training stages. Even with more parameters, hidden states still cannot catch up.}
  \label{fig:predict_accuracy}
\end{figure*}

\section{Non-support Samples}\label{sec:non_support}

In Section 2, from the counterfactual (removing non-support samples), we preliminarily infer that non-support samples are crucial for learning model parameters beyond the heads. Now we further show that non-support samples play a role in learning representations. 

In Transformers architectures, higher-layer representations are typically better than lower-layer ones. Therefore, if higher-layer representations “contain” more non-support samples, could this indicate that non-support samples contribute to learning representations? We probe each layer of LM (learning new heads for each hidden states with the original training object of the LM). 

Figure \ref{fig:hiddenprob_nonsupport} describes the number and POS distribution of non-support samples across different layers.\footnote{We additionally present the probing result of memorized samples in Appendix \ref{section:layer-wise}, which as a subset of non-support samples, show similar outcomes to non-support.}
Overall, representations of the LM "contain" more non-support samples at higher layers. In the first six layers, the number of non-support samples remains relatively low and stable, with a sudden surge observed from the fifth to the sixth layer. The growth in higher layers is markedly more pronounced than in lower layers. For POS, the LM shows some recognition of NUM in the first layer. Notably, PUNCT is heavily contained at the sixth layer.

Now we can answer: non-support samples indeed play a role in learning representations. Conversely, they can also be viewed as a visualization of representational capacity: the LM's representational ability undergoes a qualitative leap starting at the sixth (half) layer.


\begin{figure}[H]
  \includegraphics[width=\columnwidth]{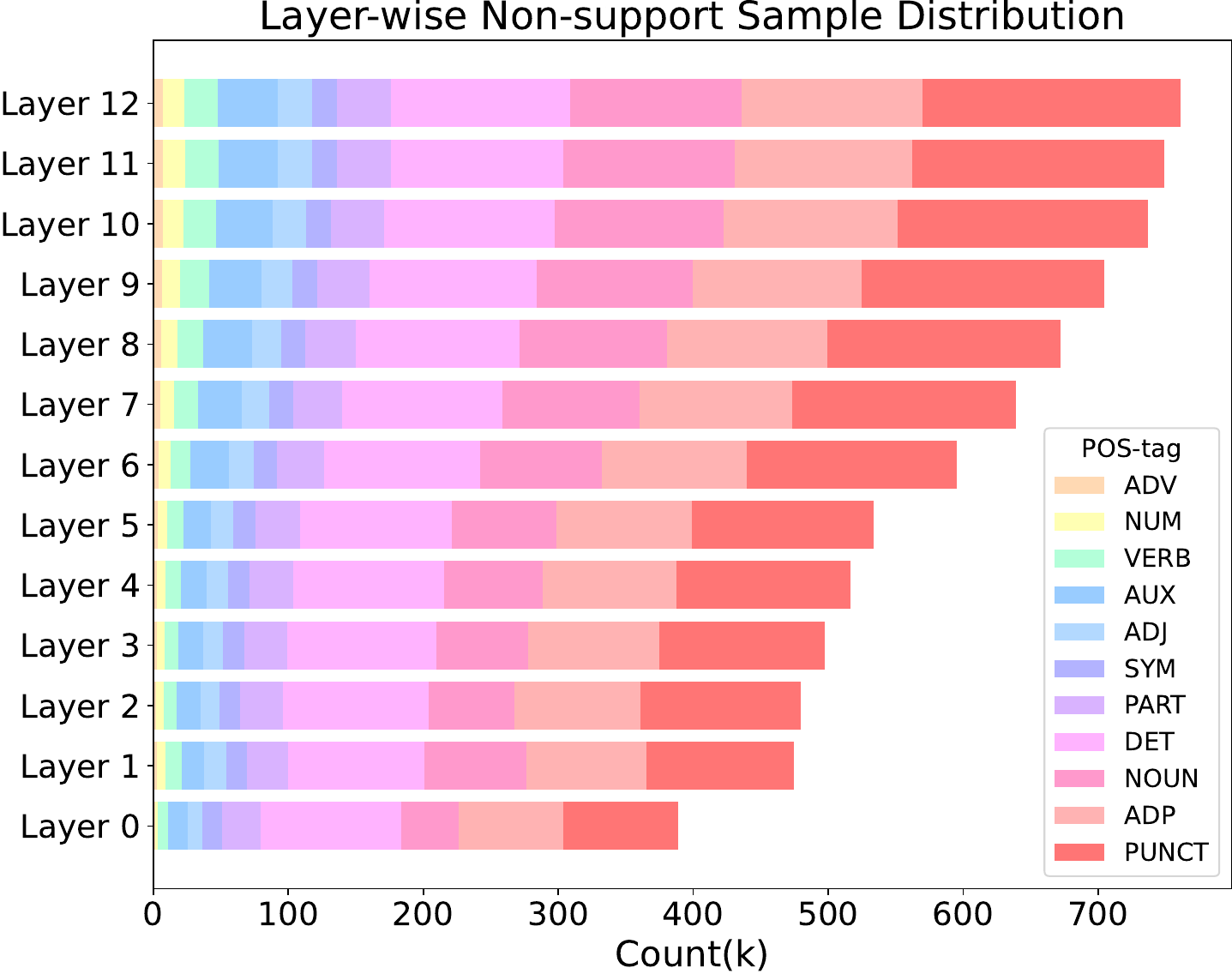}
  \caption{The number and POS distribution of non-support samples across different layers of the LM. Overall, representations of the LM "contain" more non-support samples at higher layers. In the first six layers, the number of non-support samples remains relatively low and stable, with a sudden surge observed from the fifth to the sixth layer. The growth in higher layers is markedly more pronounced than in lower layers. For POS, the LM shows some recognition of NUM in the first layer. Notably, PUNCT is heavily contained at the sixth layer.}
  \label{fig:hiddenprob_nonsupport}
\end{figure}
\section{Related Work}

\paragraph{Difficult instances and model training}

Many studies have investigated the relationship between 
the difficulty of training instances and their role for model's learning, 
and observe that keeping harder instances and a small amount of easy instances 
can maintain or even improve model's generalization performance. 
A wide range of difficulty measures have been adopted, 
such as ``forgetting" behavior, 
i.e., after which this training example is always correctly predicted~\citep{toneva2018an}, 
training loss magnitude~\citep{pmlr-v139-jiang21k,swayamdipta-etal-2020-dataset}, 
gradient norm~\citep{paul2021deep}, 
the number of layers needed for prediction~\citep{baldock2021deep}, 
perplexity~\citep{kwok2024dataset}, 
and hidden representations' distance to other instances~\citep{sorscher2022beyond}. 
Different from these works that aim to identify small but informative subsets for training, 
we study the specific roles of support and non-support examples 
in learning representations and decision boundaries.

\paragraph{Predictive Data Attribution}

Predictive data attribution seeks to answer the question: 
what would have happened if we had trained on a different dataset?~\citep{madry2024dataattribution}. 
Most existing approaches rely on the gradients of individual instances.\footnote{
A notable exception is the work of \citet{ilyas2022datamodels}, 
which uses a linear model of instances to predict test accuracy.}
One line of research uses influence functions~\citep{pmlr-v70-koh17a} 
and its approximations~\citep{NEURIPS2019_a78482ce,park2023trak,grosse2023studyinglargelanguagemodel}. 
However, other studies have observed them to be fragile~\citep{basu2021influence} 
and do not answer the question of leave-one-out retraining~\cite{bae2022if}. 
The other line of works approximates the training dynamics, 
by unrolling the updates during each step~\citep{
NEURIPS2019_5f146156,NEURIPS2020_e6385d39,bae2024training}. 

\section{Conclusion}

This work has taken a data-centric approach, 
to study the training samples that contribute the most to the prediction of each token 
using a representation theorem. 
Concretely, we make three contributions,

\begin{itemize}
    \item We observe that support samples constitute over half of the training set.
    Interestingly, some predictions are supported by a large number of samples,
    while others are supported by only a few (\S\ref{sec:first_glance}).
    suggesting language models may employ two distinct modes of prediction. 
    Furthermore, we identify two types of support samples:
    those that help predict the correct tokens and those that prevent incorrect token predictions.

    \item We further study the roles of support samples in training 
    using removing-retraining experiments. 
    We find that support and non-support samples are respectively more useful 
    for learning prediction heads and backbone representations
    (\S\ref{sec:importance_of_support} and \S\ref{sec:non_support}). 
    Interestingly, 
    we observe that a few non-support samples are always needed for successful training. 

    \item Inspired by the observation that 
    non-support samples do not contribute to learning support samples, 
    we investigate whether the difficulty of support samples is intrinsic 
    by predicting them without training language models.
    Surprisingly, we observe that 
    using only the gradients of a randomly initialized language model 
    can predict support samples with over 80\% accuracy, 
    validating our hypothesis (\S\ref{sec:predicting_support}). 

\end{itemize}
\section{Limitation}

\paragraph{On the ambiguity between support and non-support samples.}
In section~\ref{sec:non_support},
we have noted that even when using features from the last checkpoint for training, 
the classifier’s accuracy on the test set is still not perfect. 
This suggests that 
the boundary between support and non-support samples is not well-defined. 
Directly dividing training samples into support and non-support 
based on a threshold may have inherent issues. 

\paragraph{Beyond argmax: rethinking decoding and sample importance.} 
The argmax decoding enters the representation theorem through the MLE loss function(Theorem \ref{thm:rep}), which is common in LM training. If the language model is trained in a different object function, both the representation of next-word prediction heads and the definition of important samples could be different.

However, for inference, the analysis is not limited to the argmax decoding. For example, when estimating importance scores, one could use sampling instead of argmax prediction. The sampling-based decoding may better reflect the nature of language data, which often does not align with a deterministic majority-vote interpretation. This may help address the concern that the observed next token $y$ should not be assumed to be the argmax outcome of the underlying distribution.

In terms of the sparsity of next-word distribution and the role of argmax decoding. We would think that the learning of language model tries to establish an approximate mapping between context-dependent semantics of a token $v$ (i.e., presentations of various prefixes ended with $v$) and context-independent semantics of the token (i.e., presentations of $v$ in vocabulary). We would also think that the next-word prediction task is more like a nearest-neighbor search in the learned vocabulary space.

\section{Acknowledgement}
The authors wish to thank all reviewers for their helpful comments and suggestions. The corresponding authors are Yuanbin Wu and Yufang Liu.

\bibliography{main}

\appendix

\section{Proof of Theorem \ref{thm:rep}}
\label{section:proof}

\begin{proof}
We first compute derivatives of $\log p(\mathbf{y}|\mathbf{x})$,
\begin{equation*}
   \diff{\log p(\mathbf{y}|\mathbf{x})}{\theta_v} = 
   \diff{\left(\theta_\mathbf{y}^T\phi(\mathbf{x}) - \log Z \right)}{\theta_v}.
\end{equation*}
For the normalizer
$Z=\sum_{v'}\exp{\left(\theta_{v'}^T\phi(\mathbf{x})\right)}$,
we have,
\begin{equation*}
    \diff{\log Z}{\theta_v} = \frac{1}{Z}\exp{\left(\theta_v^T\phi(\mathbf{x})\right)}\cdot\phi(\mathbf{x})
    = p(v|\mathbf{x})\cdot\phi(\mathbf{x}).
\end{equation*}
Therefore, 
\begin{align*}
   \diff{\log p(\mathbf{y}|\mathbf{x})}{\theta_v} &= 
   \left\{\begin{array}{cc}
    \left(1-p(\mathbf{y}|\mathbf{x})\right)\phi(\mathbf{x}), 
    & \mathbf{y}=v \\
    -p(\mathbf{y}|\mathbf{x})\phi(\mathbf{x}),
    & \mathbf{y}\neq v 
   \end{array}
   \right. \\
   &= \left(\mathbbm{1}(\mathbf{y}=v)-p(\mathbf{y}|\mathbf{x})\right)\phi(\mathbf{x}).
\end{align*}

Next, since $\theta$ is a stationary point, 
$\diffp[]{L(\theta,D)}{\theta_v}=0$, we have
\begin{align*}
   -\frac{1}{N}
   \sum_{i=1}^N 
   (\mathbbm{1}(\mathbf{y}_i=v)-p(v|\mathbf{x}_i))\phi(\mathbf{x}_i)
    + 2\lambda\theta_v = 0 \\
    \Rightarrow
   \theta_v = \frac{1}{2N\lambda}
    \sum_{i=1}^N
   (\mathbbm{1}(\mathbf{y}_i=v)-p(v|\mathbf{x}_i))\phi(\mathbf{x}_i)
\end{align*}
\end{proof}

\section{Rationale Behind the Choice of Hyperparameter $\tau$}
\label{section:tau}
We set $\tau=0.9$ empirically, based on the following rationale:
\paragraph{Distribution of sample importance scores $\alpha$.} 
As shown in Figure \ref{fig:alpha_dis}, we observe a drastic increase of $\alpha$ from $9.27\%$ to $54.16\%$ from $\tau=0.8$ to $\tau=0.9$, suggesting a natural threshold for different regions.
\paragraph{Ablation studies for threshold parameter $\tau$ sensitivity.} 
We report the number of support samples with different $\tau$ as illustrated in Figure \ref{fig:thre_sen}.
The sharp decrease of the number of support samples from $0.9$ to $1.0$ further justifies our choice.

\paragraph{Softmax sharpen probabilities.}
Since probabilities are processed through softmax, they tend to be quite sharp. Actually, using a temperature parameter can make the model less sensitive to threshold choices, resulting in a smoother distribution.

\begin{figure}[t]
  \includegraphics[width=\linewidth]{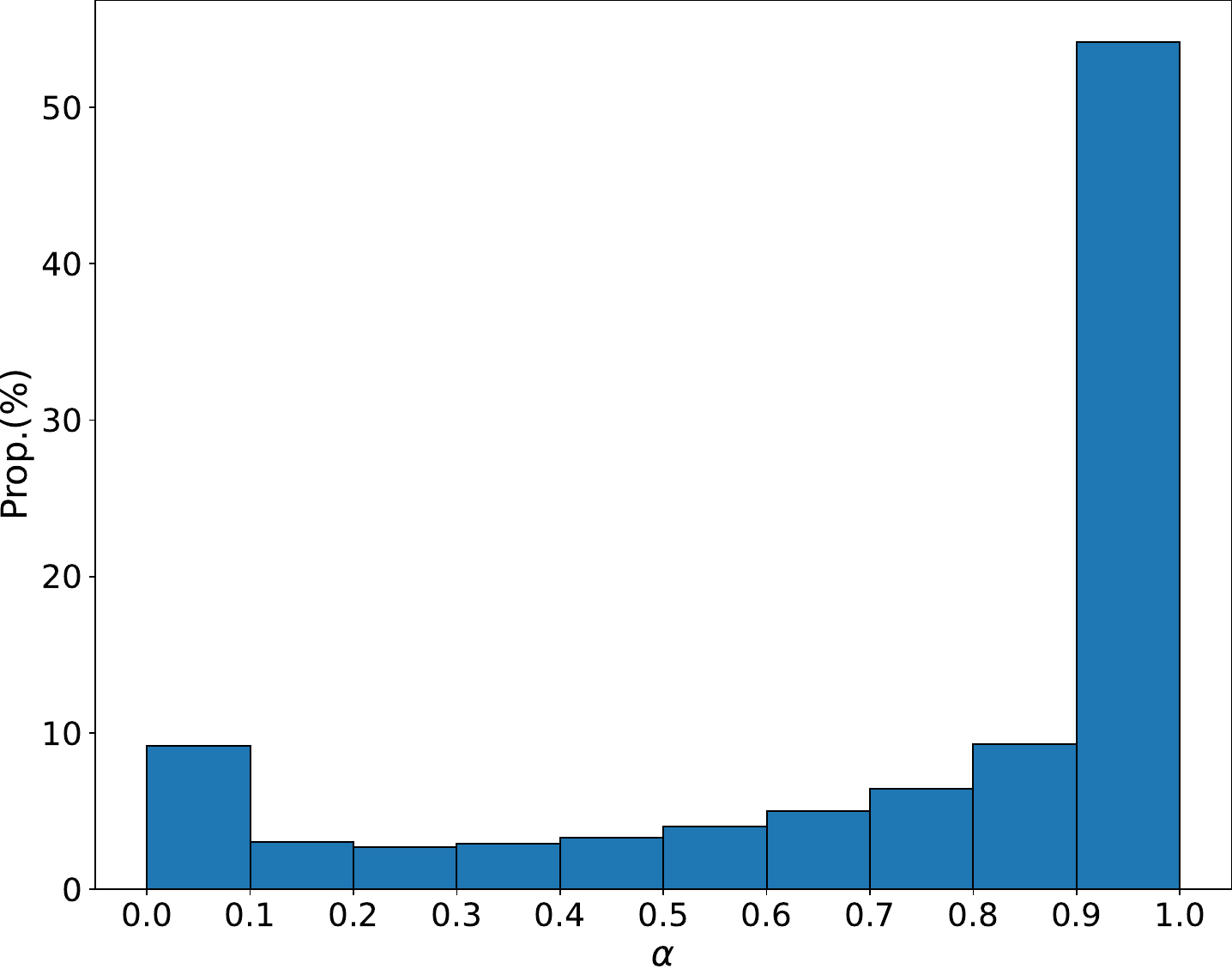}
  \caption {Distribution of sample importance scores $\alpha$.}
  \label{fig:alpha_dis}
\end{figure}

\begin{figure}[t]
  \includegraphics[width=\linewidth]{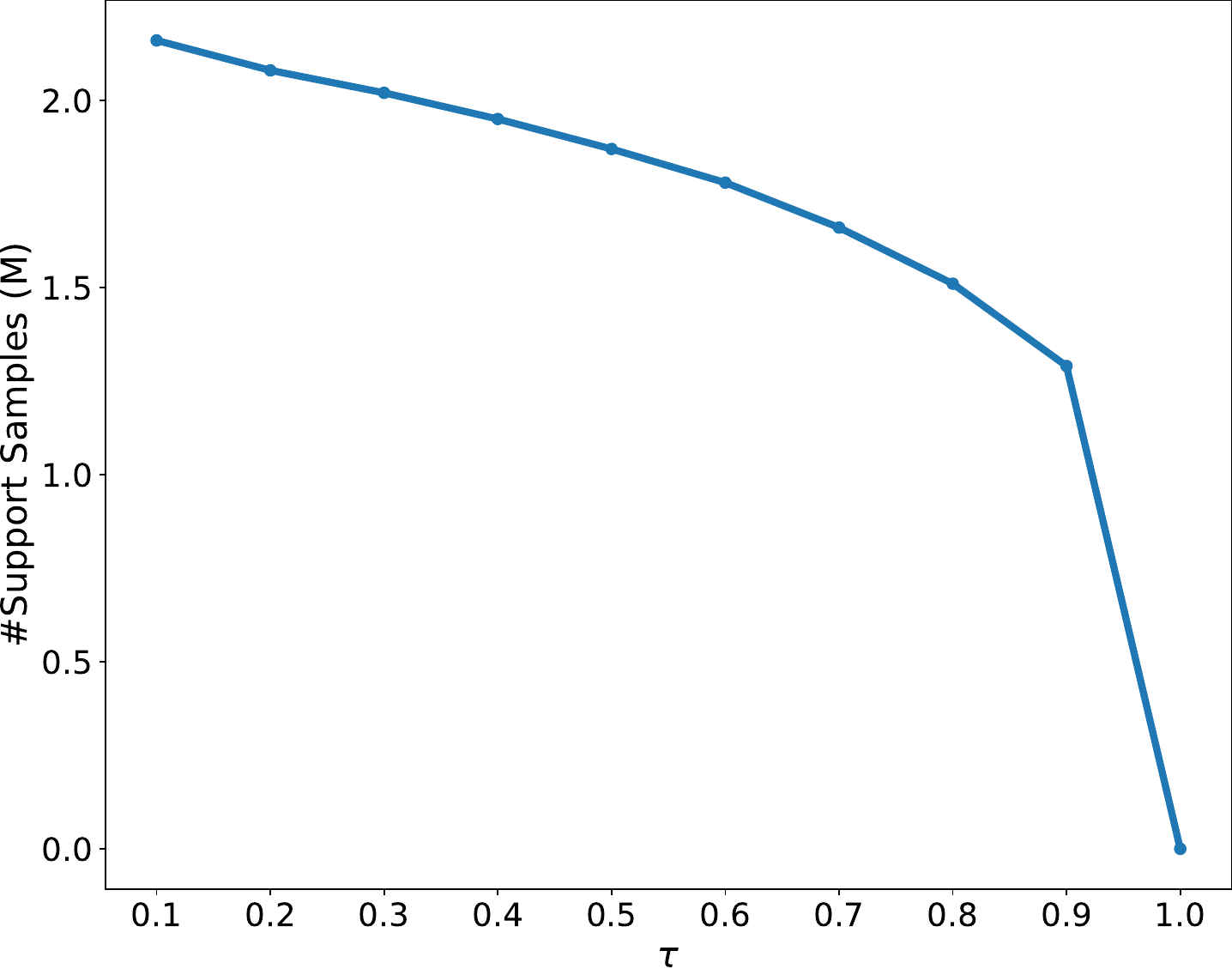}
  \caption {The number of support samples with different $\tau$.}
  \label{fig:thre_sen}
\end{figure}

Actually, the choice of threshold does not affect the analysis framework in the work: all methods can be generalized to any threshold.

On the other hand, threshold selection is practically significant. From a practitioner’s viewpoint, with a fixed budget, determining an appropriately sized support set and exploring how to adjust thresholds for different support set sizes are valuable directions for future work.

\section{Concrete Example}
\label{section:example}
We present some support and non-support samples of \textit{etts}, where the target token is \textit{etts} itself, in Table \ref{tab:example}.
In our framework, \textit{etts} is typically a suffix observed in proper nouns such as \textit{Burnetts}, \textit{Corbetts}, and \textit{Plunketts}.
Our statistics reveal that \textit{Burnetts} and \textit{Corbetts}(of support samples) appear once each in the training data, while \textit{Plunketts}(of non-support samples) appears 68 times.
This example suggests that less frequent patterns in the training data are more likely to become support to the LM heads' parameters.

\begin{table}[h]
\centering
\small
\begin{tabular}{@{}p{0.2\linewidth}p{\dimexpr 0.8\linewidth-2\tabcolsep}@{}}
\toprule
\textbf{Type} & \textbf{Sample} \\
\midrule
\multirow{2}{*}{Support}
& \parbox[t]{\linewidth}{\raggedright\textit{... The northern edge of the Plunketts Creek drainage basin is formed by \underline{Burnetts}}} \\
\cmidrule(l){2-2}
& \parbox[t]{\linewidth}{\raggedright\textit{... They are mainly composed of granite that has weathered into more rounded hills with many long scree slopes on their flanks. The highest point of these hills is Glamaig, one of only two \underline{Corbetts}}} \\
\midrule
Non-support 
& \parbox[t]{\linewidth}{\raggedright\textit{... Much of the Plunketts Creek valley is composed of various glacial deposits, chiefly alluvium. Although the \underline{Plunketts}}} \\
\bottomrule
\end{tabular}
\caption{Partial support and non-support samples of the token \textit{etts}.}
\label{tab:example}
\end{table}

\section{Layer-wise Non-support/Memorized Samples Distribution}
\label{section:layer-wise}
Figure \ref{fig:memorized} presents the probing result of memorized samples, which reveals a trend consistent with non-support samples as shown in Figure \ref{fig:hiddenprob_nonsupport}.

\begin{figure}[t]
  \includegraphics[width=\linewidth]{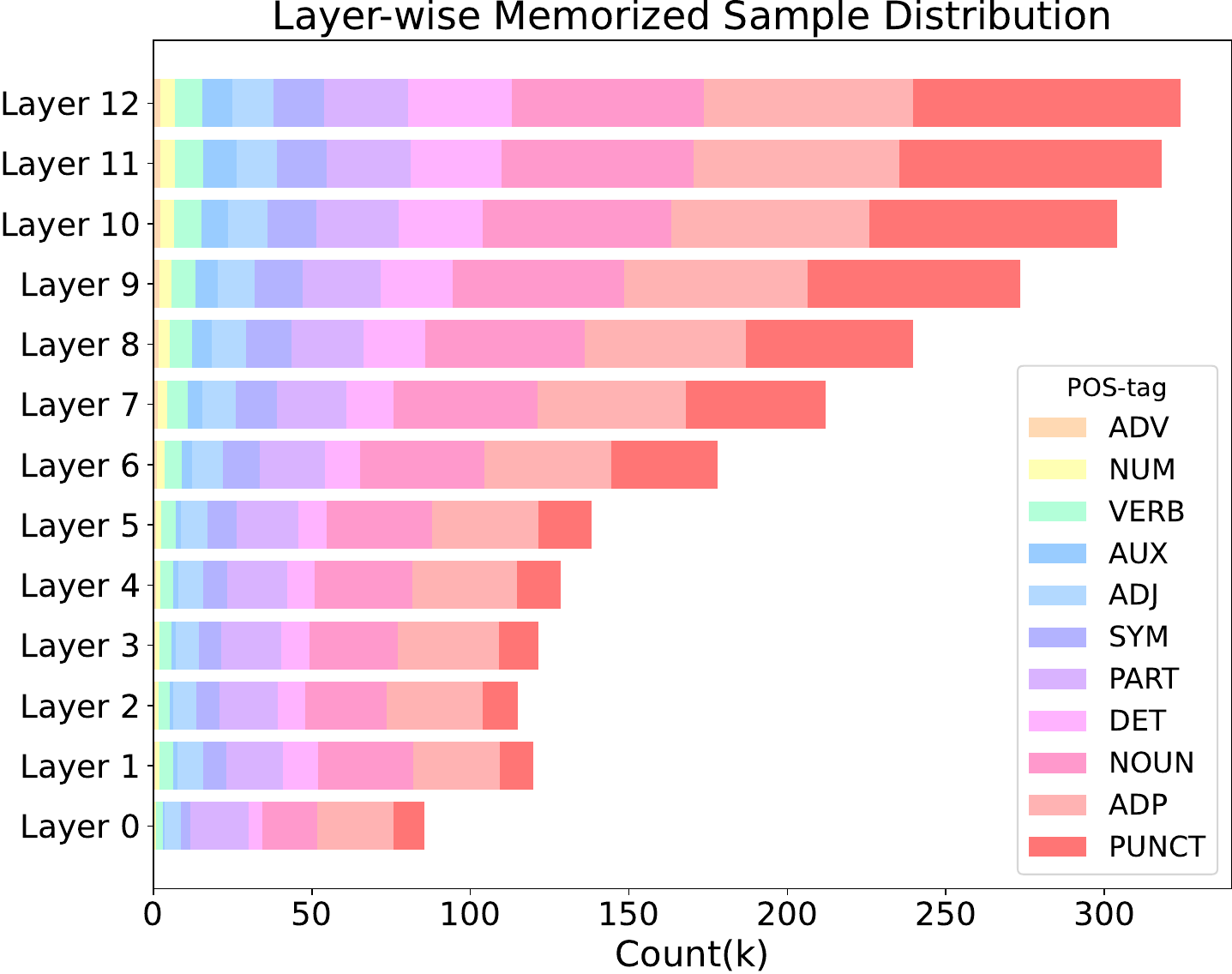}
  \caption {The number and POS distribution of memorized samples across different layers.}
  \label{fig:memorized}
\end{figure}

\section{Scaling Up}
\label{section:scaleup}

\begin{table*}[t]
  \centering
  \begin{subtable}[t]{\linewidth}
  \small
    \centering
    \begin{tabular}{lccc}
      \toprule
      \multirow{2}{*}{\textbf{Method}} & \multirow{2}{*}{\textbf{ Test Loss}} & \multicolumn{2}{c}{\textbf{\# Support Samples(M)}} \\
      \cmidrule(lr){3-4}
      & & new training set (after removal) & original training set \\
      \midrule
      \textbf{-}    & 5.02    & 1.13 &    1.13   \\
      hard    & 5.62      & 1.13$\to$0.98&  1.13$\to$1.77  \\ 
      soft    & 5.07   &  0.73$\to$0.70&   1.13$\to$1.25   \\
      random    & 5.12   & 0.57$\to$0.53 &   1.13$\to$1.23 \\\bottomrule
    \end{tabular}
    \caption{only training LM heads}
    \label{tab:first}
  \end{subtable}
  \hspace{0.05\linewidth}
  \begin{subtable}[t]{\linewidth}
    \centering
    \small
    \begin{tabular}{lccc}
      \toprule
      \multirow{2}{*}{\textbf{Method}} & \multirow{2}{*}{\textbf{Test Loss}} & \multicolumn{2}{c}{\textbf{\# Support Samples(M)}} \\
      \cmidrule(lr){3-4}
      & & new training set (after removal) & original training set  \\
      \midrule
      \textbf{-}      & 5.02   & 1.13 &   1.13    \\
      hard      &6.67     &  1.13$\to$1.12 &    1.13$\to$2.36 \\
      soft      &5.77   &  0.87$\to$0.70&   1.13$\to$1.78    \\
      random      &5.51  & 0.65$\to$0.53&  1.13$\to$1.51  \\\bottomrule
    \end{tabular}
    \caption{full model training}
    \label{tab:second}
  \end{subtable}
  \caption{Performance of the predictor and Change of support samples of the $774$M model
  when removing non-support samples from original training set with various sampling methods.}
  \label{tab:removing non-support scaleup}
\end{table*}
\begin{figure}[t]
  \includegraphics[width=\linewidth]{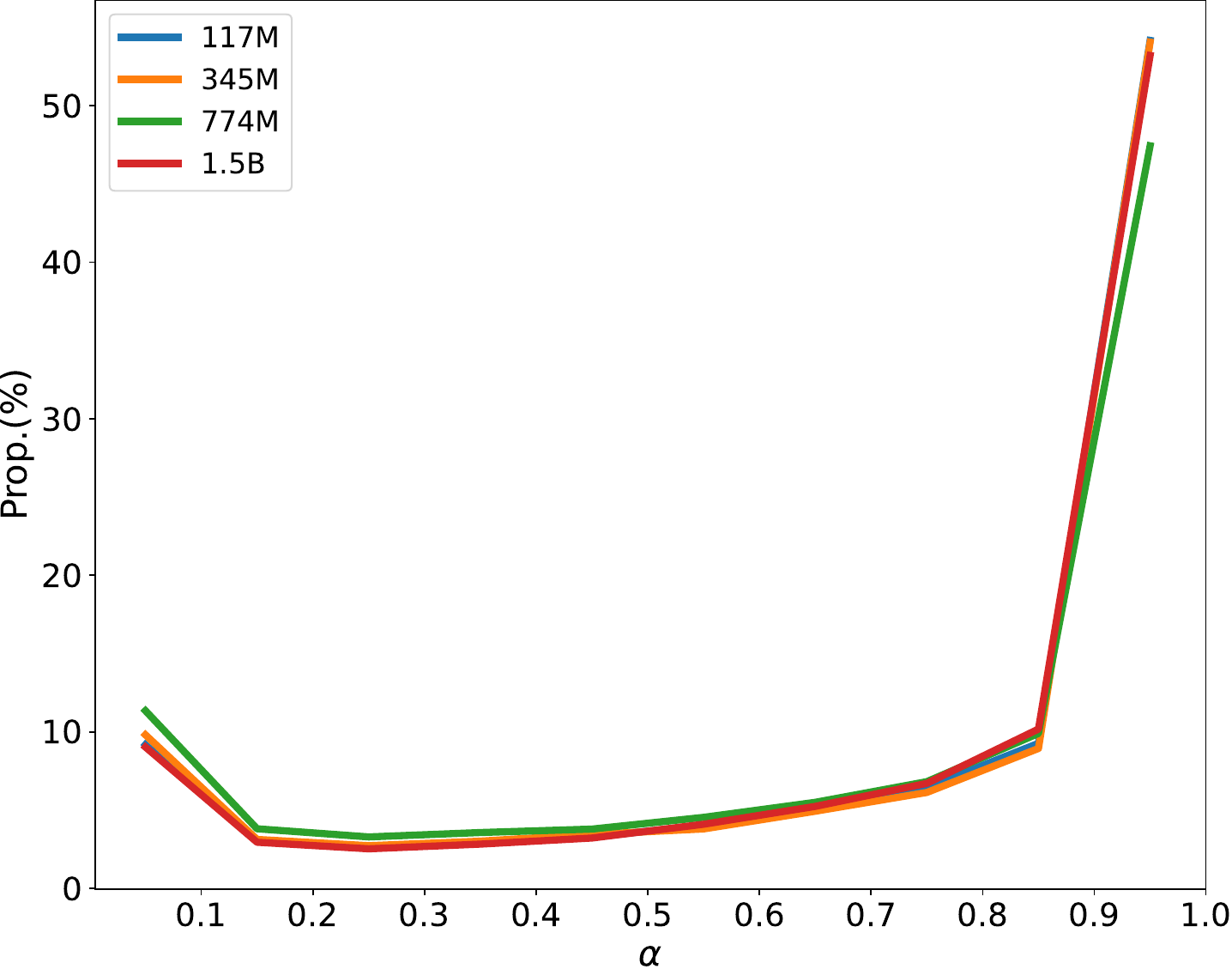}
  \caption {Distribution of samples importance scores $\alpha$ across models of various sizes trained on the wikitext-2 dataset.}
  \label{fig:alpha_dis_scaleup}
\end{figure}
\begin{figure}[t]
  \includegraphics[width=\linewidth]{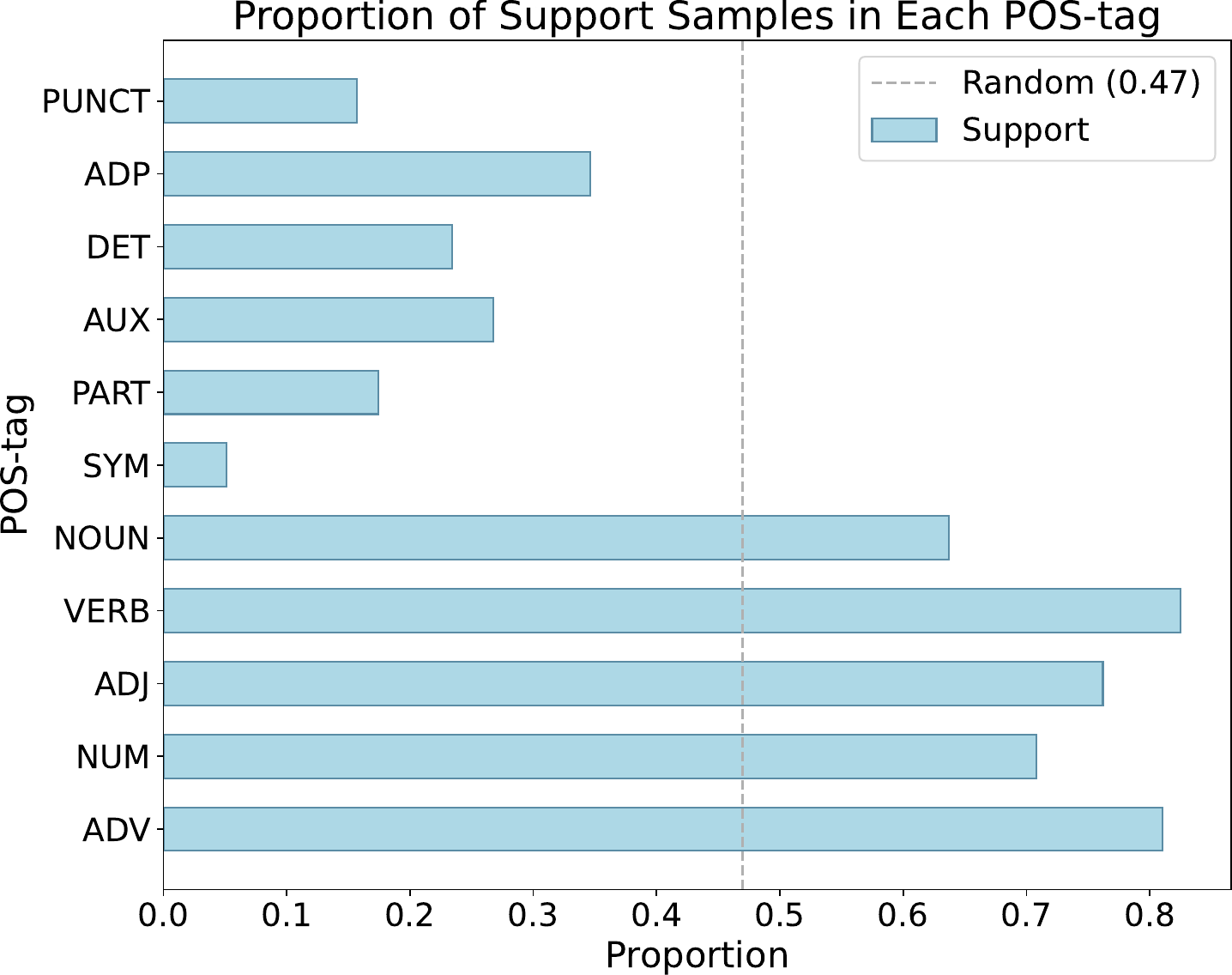}
  \caption {Proportion of support samples in different POS tags of $774$M model.}
  \label{fig:alpha_dis_scaleup}
\end{figure}
We scale our experiments to $345$M(GPT-2-Medium architecture), $774$M(GPT-2-Large architecture), and $1.5$B(GPT-2-XL architecture) models. 
Our findings remain consistent across these scales as follows.
\paragraph{Distribution of sample importance scores $\alpha$.} 
We compare the 
distribution across models of various sizes trained on the wikitext-2 dataset and find them similar, highlighting the scalability of our method.

\paragraph{Number of support samples.} 
Table \ref{tab:number_scaleup} lists the number of support samples for models of various sizes trained on the wikitext-2 dataset, showing consistency across scales.

\begin{table}[h]
\centering
\small
\begin{tabular}{lcc}
\toprule
\textbf{Model Size} & \textbf{Eval Loss} & \textbf{\#Support Samples(Prop.)} \\
\midrule
$117$M  & $5.05$ & $1.29$M ($54\%$) \\
$345$M  & $5.01$ & $1.28$M ($54\%$) \\
$774$M  & $5.02$ & $1.13$M ($47\%$) \\
$1.5$B  & $5.13$ & $1.26$M ($53\%$) \\
\bottomrule
\end{tabular}
\caption{Evaluation loss and number of support samples for models of various sizes trained on the wikitext-2 dataset.}
\label{tab:number_scaleup}
\end{table}

We further replicate most of results using $774$M model with the wikitext-2 dataset. Following are some key observations.

\paragraph{POS tags of support samples.} 
Figure \ref{fig:alpha_dis_scaleup} shows that for $774$M model, $82\%$ of verbs are support samples while the proportion is only $16\%$ for punctuations. This feature is consistent with the $117$M model.

\paragraph{Removing non-support samples.} 
We conduct the experiment of removing non-support samples on the $774$M model. As shown in Table \ref{tab:removing non-support scaleup}, the conclusions(\S\ref{sec:importance_of_support}) still hold at this larger scale.
For only training LM heads, we can remove non-support samples(around $50\%$ of the samples) while keeping the same performances, with weighted sampling. Following the same reasoning(\S\ref{sec:importance_of_support}), we get for tuning heads, support samples are key, but some non-support samples are also necessary. 
For full model training, we can’t simply remove non-support samples. Soft sampling outperforms random sampling when training LM heads, but the opposite is true for full model training, suggesting non-support samples are crucial for learning representations beyond the heads.

\end{document}